\documentclass[11pt]{article}
\usepackage[top=1in,left=1in,bottom=1in,right=1in]{geometry}
\usepackage{amsmath,amsthm,amssymb}
\usepackage{cite}
\usepackage{bbm}
\usepackage{xcolor}
\usepackage{graphicx, color}
\usepackage{graphics}
\usepackage{comment}
\usepackage{subcaption}
\usepackage{caption}
\usepackage{setspace}
\newtheorem{theorem}{Theorem}
\newtheorem{proposition}{Proposition}

\DeclareMathOperator*{\argmin}{arg\,min}

\author{Mario Diaz\footnote{Instituto de Investigaciones en Matem\'{a}ticas Aplicadas y en Sistemas (mario.diaz@sigma.iimas.unam.mx)} \and Peter Kairouz\footnote{Google AI (kairouz@google.com)} \and Jiachun Liao\footnote{Arizona State University (Jiachun.Liao@asu.edu;lsankar@asu.edu)} \and Lalitha Sankar${}^\ddagger$}
\title{Theoretical Guarantees for Model Auditing with Finite Adversaries}
\date{\today}

\begin{document}

\maketitle

\begin{abstract}
  Privacy concerns have led to the development of privacy-preserving approaches for learning models from sensitive data. Yet, in practice, even models learned with privacy guarantees can inadvertently memorize unique training examples or leak sensitive features. To identify such privacy violations, existing model auditing techniques use \textit{finite adversaries} defined as machine learning models with (a) access to some finite side information (e.g., a small auditing dataset), and (b) finite capacity (e.g., a fixed neural network architecture). Our work investigates the requirements under which an unsuccessful attempt to identify privacy violations by a finite adversary implies that no stronger adversary can succeed at such a task. We do so via parameters that quantify the capabilities of the finite adversary, including the size of the neural network employed by such an adversary and the amount of side information it has access to as well as the regularity of the (perhaps privacy-guaranteeing) audited model.
\end{abstract}

\section{Introduction}

The success of machine learning algorithms hinges on the availability of large datasets that often contain sensitive information about participating individuals, thus introducing substantial privacy risks. One way to minimize these risks is to store less data: providers can methodically consider what data to collect and how to log it. However, even when such datasets are handled with care and stored anonymously, user information can still leak through deployed machine learning models that are trained on anonymous datasets containing hidden correlations with sensitive information. Indeed, recent works show that ``overparameterized'' machine learning models have the tendency to memorize unique training examples and overlearn sensitive attributes \cite{fredrikson2015model, carlini2018secret, jagielski2019high, melis2018exploiting, nasr2018comprehensive, shokri2017membership, song2019auditing, song2019overlearning, song17remember}. 

These concerns have led to a growing body of research focused on discovering techniques that provide rigorous privacy guarantees. This, in turn, led to the development of differential privacy (DP): a mathematically rigorous notion of privacy that prevents such memorization by placing an upper bound $\varepsilon$ on the worst-case information loss \cite{dwork2006calibrating, dwork2006our, dwork2008differential}. However, depending on the application, a particular choice of $\varepsilon$ may give strong or weak privacy guarantees. Moreover, implementations of differentially private machine learning algorithms often select large values of $\varepsilon$ in order to get acceptable utility, with little understanding of the impact of such choices on meaningful privacy. Thus, DP implementations can be complemented by direct measurement of memorization, as in \cite{carlini2018secret, song2019auditing, jayaraman2019evaluating}. Further, while often unrealistic as a threat model, black-box and white-box membership inference attacks can also be used as a tool to empirically quantify memorization \cite{hayes2019logan, hilprecht2019reconstruction, shokri2017membership}.

Techniques for quantifying unintended memorization and feature leakage rely on ``attacking" learned machine models via \textit{finite adversaries} defined as machine learning models with (a) access to some finite side information (e.g., a small auditing dataset), and (b) finite capacity (e.g., a fixed architecture such as a neural network with a fixed number of layers and neurons per layer). When such attacks are successful, one can safely deduce that the audited model has memorized a user's unique training example and/or is capable of leaking sensitive attributes \cite{fredrikson2015model,carlini2018secret, jagielski2019high, melis2018exploiting, nasr2018comprehensive, shokri2017membership, song2019auditing, song2019overlearning, jayaraman2019evaluating, song17remember}. However, we cannot certify that a learned model is ``privacy-preserving’’ when these attacks fail at finding privacy violations. This is because it is still possible that a stronger adversary (with larger capacity and/or access to more side information) can succeed at identifying such violations. In this work, we identify precise conditions under which an unsuccessful attack by a finite adversary implies that no (stronger) adversary can perform a successful attack.
%%%

\subsection{Contributions}

We ask the following fundamental question. Assume that a finite adversary with limited capabilities fails at finding unintended information leakage by a target machine learning model $g$. Under what circumstances is it possible to guarantee that no other adversary will succeed at finding (and exploiting) such a leakage? We consider a general setup that captures many instances of the above situation. As shown in Section~\ref{Section:Motivation}, it includes both \textit{unintended memorization}, where an adversary would like to infer whether or not the data of a particular individual was used to train $g$, and \textit{unintended feature leakage}, where an adversary would like to learn a sensitive attribute from the output of $g$.

To measure the unintended information leakage of a model $g$, a machine learning auditor designs a classifier $h$ (referred to as adversary) that attempts to recover sensitive information from $g$. Ideally, the modeler would like to ensure that the minimal true loss (minimum under all possible adversaries $h$) is large, indicating that no $h$ can reliably recover sensitive information from $g$. However, such a quantity cannot be computed in practice because the population distribution is unknown and it is (computationally) impossible to implement all adversaries.
To circumvent this problem, we provide lower bounds for the minimal true loss over all possible adversaries using the minimal empirical loss of a two-layer neural network adversary. Our setting differs from classical (PAC) learning settings in the following three ways: (i) the minimal true loss is defined over all (measurable) functions in order to account for every possible adversary; (ii) it is desirable for this true loss to be as large as possible; and (iii) lower bounds for the minimal true loss are sought to guarantee that every adversary is restricted in its learning. Recall that in classical learning settings: (i) the minimal true loss is defined over a restricted hypothesis class; (ii) it is desirable for the true loss of the empirical risk minimizer (ERM) to be as small as possible; and (iii) upper bounds on the true loss of the ERM are sought to guarantee that the ERM loss is close to the minimal true loss for the hypothesis class. 

In order to obtain our main bounds, we combine generalization and function approximation results from learning theory, along with statistical properties of the two losses we consider: square-loss and log-loss. It is important to remark that our bounds depend heavily on the Barron constant, a quantity used to quantify the approximating power of neural networks \cite{barron1993universal,lee2017ability}.

In a classification setting where $g$ is a linear classifier and the data comes from a multi-dimensional Gaussian mixture distribution, we provide an integral representation for the minimal true loss and compute the Barron constant explicitly. In particular, we establish the precise dependency of the Barron constant on the learning model $g$. We then provide a lower bound for the minimal true loss using the minimal empirical loss of a two-layer neural network adversary and the Barron constant. As is typical for large deviation-based generalization results, our bounds are not tight but they are meaningful in some situations. Indeed, when $g$ is privacy preserving, its Barron constant is small, which leads to a tight gap between the minimal true loss and the minimal empirical loss.
%%%

\subsection{Related Work}
There is a wealth of emerging research on model auditing as well as on techniques to evaluate their strengths via adversarial models and attacks \cite{fredrikson2015model, carlini2018secret, jagielski2019high, melis2018exploiting, nasr2018comprehensive, shokri2017membership, song2019auditing, song2019overlearning}. We capture unintended memorization using the \textit{membership inference attacks} framework by \cite{shokri2017membership} where the authors determine if a user's data was used in the training of an audited model by training an adversarial inference model to recognize differences in the audited model's predictions on the inputs that it trained on versus inputs that it did not train on. We capture unintended feature leakage using the adversarial attack model by \cite{song2019overlearning, huang2018generative, huang2017context} where the authors use an auxiliary dataset (representing side information) to train an adversarial inference model capable of predicting sensitive attributes from the output of the audited model.  

From a technical perspective, our results rely on generalization and function approximation results. The latter results characterize how well certain functions can be approximated using neural networks, see \cite{cybenko1989approximation,hornik1989multilayer,barron1993universal,eldan2016power,telgarsky2016benefits,daniely2017depth,lee2017ability,mehrabi2018bounds} and references therein. In particular, our main results heavily rely on the seminal work of \cite{barron1993universal} regarding the rate (with respect to the number of neurons) at which 2-layer neural networks approximate certain functions. Regarding generalization, namely how different in-sample and out-of-sample errors are, our work relies on classic bounds as those that can be found, for example, in \cite{shalev2014understanding}. More contemporary developments in this topic can be found in \cite{bousquet2002stability,shalev2010learnability,shamir2010learning,russo2015much,zhang2016understanding,xu2017information} and references therein. The work of \cite{shamir2010learning} is closely related to our results under log-loss. More specifically, we use similar continuity properties for (conditional) entropy where applicable; however, we rely on more contemporary results in \cite{alhejji2019tight}.
%%%%%%%

\section{Preliminaries}

In this section we gather the main notation and results used in this paper. In particular, this section summarizes classical function approximation results and some continuity properties of the conditional entropy.

For $d\in\mathbb{N}$, we define $[d] = \{1,\ldots,d\}$. For $p\in\mathbb{N}$ and $x = (x_1,\ldots,x_p)\in\mathbb{R}^p$, we let $\|x\| = (x_1^2+\cdots+x_p^2)^{1/2}$. In this paper binary variables always take values on $\{\pm1\}$. Hence, by abuse of notation, we refer to soft classifiers as functions $g:\mathbb{R}^p\to[-1,1]$. For example, logistic regression classifiers are functions of the form
\begin{equation}
    g(x) = \sigma(w \cdot x)
\end{equation}
where $w\in\mathbb{R}^q$, $w \cdot x = \sum_i w_ix_i$ and
\begin{equation}
    \sigma(t) = \frac{1-e^{-t}}{1+e^{-t}} = \tanh\left(\frac{t}{2}\right).
\end{equation}

Recall that for a probability distribution $P$ over $\mathbb{R}^q$, $\alpha\geq1$ and measurable function $f:\mathbb{R}^q\to\mathbb{R}$, the $(P,\alpha)$-norm of $f$ is defined as
\begin{equation}
    \|f\|_{P,\alpha} := \left(\int_{\mathbb{R}^q} |f(x)|^\alpha {\rm d}P(x) \right)^{1/\alpha}.
\end{equation}
Also, recall that for a function $f:\mathbb{R}^q\to\mathbb{R}$, its Fourier transform, say $\hat{f}:\mathbb{R}^q\to\mathbb{R}$, is defined as
\begin{equation}
    \hat{f}(\omega) := \frac{1}{(2\pi)^{q/2}} \int_{\mathbb{R}^q} f(x) e^{-{\rm i} \omega x} {\rm d}x.
\end{equation}

Let $q\in\mathbb{N}$ be fixed. For $k\in\mathbb{N}$, let $\mathcal{H}_k$ be the set of functions $h:\mathbb{R}^q\to\mathbb{R}$ of the form
\begin{equation}
\label{eq:DefNN}
    h(x) = c_0 + \sum_{i=1}^k c_i \sigma(a_i \cdot x + b_i),
\end{equation}
for some $a_1,\ldots,a_k\in\mathbb{R}^q$ and $b_1,\ldots,b_k,c_0,\ldots,c_k\in\mathbb{R}$. In other words, $\mathcal{H}_k$ is the set of functions that can be implemented using a two-layer neural network with $k$ neurons in the hidden layer. The following proposition, due to Barron \cite{barron1993universal}, establishes, in a quantitative manner, the universal approximating capabilities of two-layer neural networks. For a given set $\mathcal{K}\subseteq\mathbb{R}^q$, its diameter is defined as $\displaystyle \textnormal{Dia}(\mathcal{K}) = \sup_{x,y\in\mathcal{K}} \|x-y\|$.

\begin{proposition}
\label{Thm:Barron}
Let $k\geq1$ and $P$ a probability distribution supported over $\mathcal{K}\subseteq\mathbb{R}^q$. If $h:\mathbb{R}^q\to\mathbb{R}$ is a smooth function\footnote{Barron's theorem holds in greater generality. However, for the purpose of this paper the present formulation suffices. We refer the reader to \cite{barron1993universal} for further details.}, then there exists $h_k\in\mathcal{H}_k$ such that
\begin{equation}
    \|h-h_k\|_{P,2} \leq \frac{\textnormal{Dia}(\mathcal{K})C_h}{\sqrt{k}},
\end{equation}
where $C_h$ is the so-called Barron constant of $h$ defined as
\begin{equation}
\label{eq:DefBarronConstant}
    C_h = \frac{1}{(2\pi)^{q/2}} \int_{\mathbb{R}^q} \|\omega\|\,|\hat{h}(\omega)| {\rm d}\omega.
\end{equation}
Furthermore, the coefficients of the linear combination in \eqref{eq:DefNN} may be restricted to satisfy $c_0 = h(0)$ and $\displaystyle \sum_{i=1}^k |c_i| \leq \textnormal{Dia}(\mathcal{K}) C_h$.
\end{proposition}

Recall that the total variation distance between two distributions $P$ and $Q$ is defined as
\begin{equation}
    \textnormal{TV}(P,Q) := \frac{1}{2} \sum_x |P(x) - Q(x)|.
\end{equation}
For discrete random variables $U$ and $V$, the conditional entropy (or equivocation) of $U$ given $V$ is given by
\begin{equation}
    H(U|V) = - \sum_{u,v} P_{U,V}(u,v) \log\left(\frac{P_{U,V}(u,v)}{P_V(v)}\right).
\end{equation}
The next proposition \cite[eq.~(7)]{alhejji2019tight}, establishes the continuity of the conditional entropy with respect to the total variation distance.

\begin{proposition}
Let $U_1,U_2$ and $V_1,V_2$ be random variables supported over finite alphabets $\mathcal{U}$ and $\mathcal{V}$, respectively. If $\theta := \textnormal{TV}(P_{U_1,V_1},P_{U_2,V_2}) \in [0,1-\frac{1}{|\mathcal{U}|}]$, then,
\begin{equation}
\label{inq:AlhejjiSmith}
    |H(U_2|V_2) - H(U_1|V_1)| \leq \theta \log(|\mathcal{U}|-1) + h_b(\theta)
\end{equation}
where $h_b(x) = -x\log(x)-(1-x)\log(1-x)$ is the so-called binary entropy function.
\end{proposition}

Due to technical conditions, oftentimes we need to truncate probability distributions. Suppose that $V$ is a real random variable. For any $r>0$, we say that a random variable $V'$ is the truncation of $V$ to the interval $[-r,r]$ if its distribution satisfies that, for all $s,t\in[-r,r]$ with $s<t$,
\begin{equation}
    \mathbb{P}(V'\in[s,t]) = \frac{\mathbb{P}(V\in[s,t])}{\mathbb{P}(V\in[-r,r])}.
\end{equation}
%%%%%%%

\section{Problem Setup and Motivation}
\label{Section:Motivation}

In this section we introduce a general setup that captures the main features of some machine learning scenarios where unintended memorization and unintended feature leakage represent an issue. We first introduce this setup in its full generality, and then we demonstrate that it captures: (a) unintended memorization via membership inference attacks \cite{shokri2017membership}, and (b) unintended feature leakage \cite{song2019overlearning}.

\subsection{General Setup}
\label{Subsection:GeneralSetup}

Let $S\in\{\pm1\}$ and $U\in\mathbb{R}^p$ be two correlated random variables. Assume that upon $U$, a random variable $T\in\mathcal{T}$ is created to form the Markov chain $S - U - T$. In this work we assume that either $\mathcal{T}=[d]$ for some $d\in\mathbb{N}$ or $\mathcal{T}\subseteq\mathbb{R}^q$ for some $q\in\mathbb{N}$. As we show in Sections~\ref{Subsection:MembershipInferenceAttacks} and \ref{Subsection:ModelAuditing} below, in applications of interest, $S$ is a (binary) sensitive variable, $U$ is a feature vector, and $T$ is a representation of the feature vector or the output of a classifier, i.e., $T = g(U)$ for a (potentially random) mapping $g:\mathbb{R}^p\to\mathcal{T}$. Our goal is to study the unintended memorization or feature leakage of $g$.

Given the vector $T$, the goal of a \emph{learner} is to design a mapping $h$ that takes $T$ as input and produces an estimate of the sensitive variable $S$ as output. For a sample $\mathcal{S}_n = \{(S_i,T_i): i\in[n]\}$, we define the empirical loss of $h:\mathcal{T}\to\mathbb{R}$ as
\begin{equation}
    L_{\mathcal{S}_n}(h) := \frac{1}{n} \sum_{i=1}^n \ell(h(T_i),S_i),
\end{equation}
where $\ell:\mathbb{R}\times\{\pm1\}\to\mathbb{R}_+$ is a given loss function. Similarly, we define the true loss of $h$ as
\begin{equation}
    L(h) := \mathbb{E}_{S,T}\left[\ell(h(T),S)\right].
\end{equation}
We consider both the squared-loss and the log-loss,
\begin{align}
    \ell_2(\hat{s},s) &:= (\hat{s}-s)^2,\\
    \ell_{\textnormal{log}}(\hat{s},s) &:= -\frac{1+s}{2}\log(\hat{s}) - \frac{1-s}{2} \log(1-\hat{s}).
\end{align}

We assume that the learner has limited expressiveness and statistical knowledge, i.e., they can only implement functions in $\mathcal{H}_k$ and evaluate the empirical loss $L_{\mathcal{S}_n}$. The minimal empirical loss attained by this \emph{finite} learner is given by
\begin{equation}
\label{eq:DefcalLkn}
    \mathcal{L}_{k,\mathcal{S}_n} := \inf_{h\in\mathcal{H}_k} L_{\mathcal{S}_n}(h)
\end{equation}
In this work, we assume that the learner can in fact minimize $L_{\mathcal{S}_n}$ over $\mathcal{H}_k$. Ideally, a finite learner fails to recover the sensitive variable $S$ from the vector $T$ when $\mathcal{L}_{k,\mathcal{S}_n}$ is large. Of course, large depends on the specific loss function and application. 

Observe that the true loss of any learner, independent of their expressiveness and statistical knowledge, is lower bounded by
\begin{equation}
\label{eq:DefcalL}
    \mathcal{L} := \inf_{h} L(h),
\end{equation}
where the infimum is over all (measurable) functions $h:\mathbb{R}^q\to\mathbb{R}$. In Section~\ref{Section:MainResults} we prove that, under certain conditions,
\begin{equation}
\label{eq:MainResultPrototype}
    \mathcal{L}_{k,\mathcal{S}_n} - \mathcal{L} \leq \epsilon
\end{equation}
for some small $\epsilon$ that depends on $k$ and $n$. Thus, if $\mathcal{L}_{k,\mathcal{S}_n}$ is large, then $\mathcal{L}$ is large as well. In other words, an unsuccessful attempt by a finite learner to reliably infer the sensitive variable $S$ implies that in fact no learner can succeed at such a task. Before establishing \eqref{eq:MainResultPrototype} rigorously, we show how the present setup captures some machine learning settings where unintended memorization and feature leakage are an issue.

{\bf Remark.} It is worth to point out that $\mathcal{L}_{k,\mathcal{S}_n}$ is the empirical loss attained by the \emph{empirical risk minimization} (ERM) rule over $\mathcal{H}_k$, see \eqref{eq:DefcalLkn}. Nonetheless, opposed to typical learning settings, $\mathcal{L}$ is \emph{not} the minimal true loss over the hypothesis class $\mathcal{H}_k$, i.e, $\displaystyle \inf_{h\in\mathcal{H}_k} L(h)$. Instead, $\mathcal{L}$ is the minimal true loss over every measurable function, see \eqref{eq:DefcalL}. This enlargement of the hypothesis class takes into account that other learners could have better expressiveness capabilities, hence being able to implement functions beyond $\mathcal{H}_k$. It is also worth to remark that, opposed to typical learning settings, we are interested in using the minimal empirical loss $\mathcal{L}_{k,\mathcal{S}_n}$ to produce a lower bound for the minimal true loss (although over an enlarged hypothesis class). As explained before, in this way the failure of a finite learner implies the failure of any other learner independently of their expressiveness and statistical capabilities.
%%%%%%%

\subsection{Unintended Memorization}
\label{Subsection:MembershipInferenceAttacks}

Consider a membership inference attack where the goal of an \emph{attacker}, the learner in the notation of Section~\ref{Subsection:GeneralSetup}, is to predict if the data of a given individual was used to train a target classifier $g$. To be more specific, we consider a membership inference attack composed by the following elements.

{\bf Population data.} Fix a population size $N\in\mathbb{N}$. We encode the population data in an $N \times (p+1)$ matrix
\begin{equation}
    \mathcal{D} = \left(\begin{matrix}x_1 & y_ 1\\ \vdots & \vdots\\ x_N & y_N\end{matrix}\right),
\end{equation}
where $x_i\in\mathbb{R}^p$ and $y_i\in\{\pm1\}$ for each $i\in[N]$. The variables $x_i$ and $y_i$ are the feature vector and the label of the $i$-th individual, respectively.

{\bf Classifier.} In the literature, there are two main types of attacks: black-box and white-box. In black-box attacks, it is assumed that we are given access to a binary classifier $g:\mathbb{R}^p\to[-1,1]$ only through the values of $g$ for some inputs $x_1,\ldots,x_n\in\mathbb{R}^p$. In white-box attacks, it is assumed that we are given access to $g$ itself. For example, when $g$ is implemented by a neural network, a white-box attack requires knowing the weights of the neural network, while a black-box attack only requires being able to evaluate $g$ for some inputs. In both cases, the classifier $g$ is trained using some of the individual records in $\mathcal{D}$. We encode which records were used to train $g$ via a vector $(s_1,\ldots,s_N)\in\{\pm1\}^N$. Namely, we set $s_i = 1$ if and only if $(x_i,y_i)$ was used to train $g$.

{\bf Attacker.} Given the data $(x_i,y_i)$ of the $i$-th individual and the confidence value $g(x_i)$, the attacker's goal is to learn the membership variable $s_i$. In other words, the attacker wants to learn whether the data of the $i$-th individual was used to train $g$ or not. The attacker achieves this goal by designing a function $h:\mathbb{R}\to\mathbb{R}$ which ideally satisfies that
\begin{equation}
    h(g(x_i)) \approx s_i.
\end{equation}
Following our standing assumptions, we assume that the attacker is capable to implement functions in $\mathcal{H}_k$ for some $k\in\mathbb{N}$ and can minimize $L_{\mathcal{S}_n}$ over $\mathcal{H}_k$ where
\begin{equation}
    \mathcal{S}_n := \left\{(s_{i_j},g(x_{i_j}) : j\in[n]\right\},
\end{equation}
for some $i_1,\ldots,i_n\in[N]$. Observe that this setting is supervised as the membership variables $s_{i_1},\ldots,s_{i_n}$ are known. Although supervised attacks might not be the most common ones in the literature, oftentimes attacks rely on techniques that emulate a supervised setting, see, e.g., 'shadow models' in \cite{shokri2017membership}.

Observe that the described membership inference attack falls into the general setup presented earlier. Namely, for a given $i\in[N]$, we let
\begin{equation*}
    S = s_i, \quad U = (x_i,y_i) \quad \textnormal{and} \quad T = g(x_i).
\end{equation*}
In this scenario, our main results provide conditions to guarantee that the failure of certain finite attackers implies that no attacker can succeed, independently of expressiveness and statistical knowledge.
%%%%%%%

\subsection{Unintended Feature Leakage}
\label{Subsection:ModelAuditing}

Consider a classification setting where the goal is to predict a label $Y\in[d]$ given a feature variable $X\in\mathbb{R}^p$. Furthermore, assume that there is a protected variable $Z\in\{\pm1\}$ which should not be used for classification purposes, e.g., gender. Despite not taking the protected variable as input, a predictor $\hat{Y} = g(X)$ with $g:\mathbb{R}^p\to[d]$ might still be biased by the protected variable $Z$ due the potential dependency between $Z$ and the feature variable $X$.

A possible goal of a \emph{model auditor}, the learner in the notation of Section~\ref{Subsection:GeneralSetup}, could be to determine whether or not the predictor $\hat{Y}=g(X)$ unintentionally leaks or rely on the protected variable $Z$, i.e., determine whether there is statistical dependency between $\hat{Y}$ and $Z$. We remark that this framework is typical in the fairness literature, where a vast array of notions of fairness have been proposed based on the statistical dependency of $\hat{Y}$ and $Z$ (see, for example \cite{VermaIWSF2018} and the references therein). For example, a predictor $\hat{Y}$ is said to satisfy \emph{demographic parity} if $\hat{Y}$ is independent of the protected variable $Z$. In this context, a model is audited to estimate the degree of dependence between the predictor $\hat{Y}=g(X)$ and the protected variable $Z$.

Note that the described model auditing could be framed into the general setup presented above by letting
\begin{equation*}
    S = Z, \quad U = X \quad \textnormal{and} \quad T = g(X).
\end{equation*}
In this scenario, our main results provide conditions to guarantee that the failure of an auditor in detecting dependency between the predictor and the sensitive variable implies that the predictor is indeed unbiased.

{\bf Remark.} Although we focused on membership inference attacks and feature leakage, the general setup introduced in this section captures other machine learning scenarios where privacy represents an issue, for example, \cite{carlini2018secret, song2019auditing}.
%%%%%%%

\section{Main Results}
\label{Section:MainResults}

Recall that we are given $S - U - T$ where $S\in\{\pm1\}$, $U\in\mathbb{R}^p$ and $T\in\mathcal{T}$ with $\mathcal{T} = [d]$ or $\mathcal{T}\subseteq\mathbb{R}^q$. We now present our main results that establish \eqref{eq:MainResultPrototype} quantitatively. We start with the case $\mathcal{T} \subseteq \mathbb{R}^q$, which we call the representation setting, and continue with the case $\mathcal{T}=[d]$, which we call the classification setting. The former setting captures the situation where, upon $U$, a continuous-valued representation $T$ is created; while the latter setting captures the situation where, upon $U$, an estimate $T = \hat{Y}$ is created.

\subsection{Representation Setting}

The next theorem is our main result for the representation setting ($\mathcal{T}\subseteq\mathbb{R}^q$). Recall that the diameter of a set $\mathcal{K}\subseteq\mathbb{R}^q$ is defined as $\displaystyle \textnormal{Dia}(\mathcal{K}) = \sup_{x,y\in\mathcal{K}} \|x-y\|$.

\begin{theorem}
\label{Thm:MainResultRepresentationSetting}
Let $\mathcal{T}\subseteq\mathbb{R}^q$ and $f_{+}:\mathcal{T}\to\mathbb{R}$ (resp.~$f_{-}$) be the conditional density of $T$ given $S=+1$ (resp.~$-1$). Assume that (a) $f_{-}$ and $f_{+}$ have the same support, say $\mathcal{K}$, (b) $\mathcal{K}$ is compact, and (c) the function $f_{-}/f_{+}$ extends smoothly to an open set containing $\mathcal{K}$. If $\delta>0$, then, for the squared-loss $\ell_2$ and all $k,n\in\mathbb{N}$, with probability at least $1-\delta$,
\begin{equation}
\label{eq:MainResultRepresentation}
    \mathcal{L}_{k,\mathcal{S}_n} - \mathcal{L} \leq C_0 \sqrt{\frac{\log(1/\delta)}{2n}} + \frac{C_1}{k} + \frac{C_2}{\sqrt{k}},
\end{equation}
where $C_0,C_1,C_2$ are constants independent of $k$ and $n$.

Moreover, for any smooth function $\eta:\mathbb{R}^q\to\mathbb{R}$ such that, for all $t\in\mathcal{K}$,
\begin{equation}
\label{eq:ExtensionEta}
    \eta(t) = \frac{\lambda f_{+}(t)-\overline{\lambda}f_{-}(t)}{\lambda f_{+}(t)+\overline{\lambda}f_{-}(t)},
\end{equation}
with $\lambda = \mathbb{P}(S=1)$ and $\overline{\lambda} = \mathbb{P}(S=-1)$, we have that
\begin{equation}
    C_0 \leq (2+ \textnormal{Dia}(\mathcal{K}) C_\eta)^2, \quad C_1 \leq (\textnormal{Dia}(\mathcal{K}) C_\eta)^2 \quad \textnormal{and} \quad C_2 \leq 4 \textnormal{Dia}(\mathcal{K}) C_\eta,
\end{equation}
where $P_T$ is the distribution of $T$ and $C_\eta$ is the Barron constant of $\eta$ defined in \eqref{eq:DefBarronConstant}.
\end{theorem}

\begin{proof}
For ease of notation, let $\Delta := \mathcal{L}_{k,\mathcal{S}_n} - \mathcal{L}$ and $h^*\in\argmin_h L(h)$ where the infimum is taken over all measurable functions $h:\mathcal{T}\to\mathbb{R}$. For the squared-loss $\ell_2(h,(T,S)) = (h(T) - S)^2$, we have, for all $t\in\mathcal{K}$,
\begin{equation}
    h^*(t) = \mathbb{E}(S|T=t) = \frac{\lambda f_{+}(t)-\overline{\lambda}f_{-}(t)}{\lambda f_{+}(t)+\overline{\lambda}f_{-}(t)}.
\end{equation}
By the assumed hypotheses, there exists an integrable function $\eta:\mathbb{R}^q\to\mathbb{R}$ that smoothly extends $h^*$, i.e, $\eta(t) = h^*(t)$ for all $t\in\mathcal{K}$. Since $\eta = h^*$ over $\mathcal{K}$, the support of the distribution of $T$, we have that $L(h^*) = \mathcal{L} = L(\eta)$. Furthermore, since $|\eta(t)| \leq 1$ for all $t\in\mathcal{K}$, we have that
\begin{equation}
\label{eq:NormEta}
    \|\eta\|_{P_T,2} \leq 1.
\end{equation}
By Barron's theorem (Prop.~\ref{Thm:Barron}), there exists $\eta_k\in\mathcal{H}_k$ such that
\begin{equation}
\label{eq:ProofRepresentationBarron}
    \|\eta_k - \eta\|_{P_T,2} \leq \frac{\textnormal{Dia}(\mathcal{K})C_\eta}{\sqrt{k}}.
\end{equation}
Moreover, if we let
\begin{equation}
    \eta_k(t) = c_0 + \sum_{i=1}^k c_i \sigma(a_i \cdot t + b_i),
\end{equation}
the coefficients $c_0,\ldots,c_k$ can be taken such that $c_0 = \eta(0)$ and $\sum_i |c_i| \leq \textnormal{Dia}(\mathcal{K}) C_\eta$. Observe that $|\eta(0)| \leq 1$ and hence
\begin{equation}
\label{eq:BoundInfNormEtak}
    \|\eta_k\|_\infty := \sup_{t\in\mathbb{R}^q} |\eta_k(t)| \leq 1 + \textnormal{Dia}(\mathcal{K}) C_\eta.
\end{equation}

By the minimality of $\mathcal{L}_{k,\mathcal{S}_n} := \inf_{h\in\mathcal{H}_k} L_{\mathcal{S}_n}(h)$,
\begin{equation}
\label{inq:RepresentationDeltaDecomposition}
    \Delta \leq L_{\mathcal{S}_n}(\eta_k) - L(\eta_k) + L(\eta_k) - L(\eta).
\end{equation}
Observe that $(S-\eta_k(T))^2 \leq (1+\|\eta_k\|_\infty)^2$. By \eqref{eq:BoundInfNormEtak}, a routine application of Hoeffding's inequality implies that, with probability at least $1-\delta$,
\begin{equation}
\label{inq:RepresentationGeneralization}
    L_{\mathcal{S}_n}(\eta_k) - L(\eta_k) \leq (2+\textnormal{Dia}(\mathcal{K})C_\eta)^2 \sqrt{\frac{2\log(1/\delta)}{2n}}.
 \end{equation}
Under the square-loss $\ell_2$, it could be proved that for any two functions $f_1,f_2:\mathbb{R}^q\to\mathbb{R}$,
\begin{equation}
\label{eq:LipschitzcalL}
    |L(f_2) - L(f_1)| \leq \|f_2-f_1\| (2 + 2\|f_1\| + \|f_2 - f_1\|),
\end{equation}
where the norms are $(P_T,2)$-norms. Thus, by plugging \eqref{eq:NormEta} and \eqref{eq:ProofRepresentationBarron} in \eqref{eq:LipschitzcalL},
\begin{equation}
\label{inq:RepresentationApproximation}
    |L(\eta_k) - L(\eta)| \leq \frac{\textnormal{Dia}(\mathcal{K})C_\eta}{\sqrt{k}} \left(4 + \frac{\textnormal{Dia}(\mathcal{K})C_\eta}{\sqrt{k}}\right).
\end{equation}
By plugging \eqref{inq:RepresentationGeneralization} and \eqref{inq:RepresentationApproximation} in \eqref{inq:RepresentationDeltaDecomposition}, the theorem follows.
\end{proof}

Although the conditions of the previous theorem might seem restrictive, they could be easily guaranteed by adding a small-variance noise to $T$ and then truncating the result. Indeed, if $\gamma>0$ and $Z \sim \mathcal{N}(0,1)$ is independent of $(S,T)$, then $(S,T')$ satisfy the hypotheses of the previous theorem with $T'$ being the truncation of $T+\gamma Z$ to the interval $[-r,r]$ for any $r>0$.

The assumptions of Theorem~\ref{Thm:MainResultRepresentationSetting} guarantee the existence of a function $\eta$ as in \eqref{eq:ExtensionEta}. However, to the best of the authors' knowledge, it unknown how to find the function $\eta$ that produces the smallest $C_\eta$.

{\bf Remark.} While the right hand side of \eqref{eq:MainResultRepresentation} decreases when $k$ increases, note that $\mathcal{L}_{k,\mathcal{S}_n}$ also decreases when $k$ increases. In fact, if $k\geq2n$, then a two-layer neural network with $k$ neurons can memorize the entire sample $\mathcal{S}_n$, leading to a trivial lower bound for $\mathcal{L}$ as $\mathcal{L}_{k,\mathcal{S}_n} = 0$. Furthermore, $\mathcal{L}_{k,\mathcal{S}_n}$ becomes harder to find when $k$ increases. Overall, this reveals that finding the $k$ that produces the best bound in \eqref{eq:MainResultRepresentation} is a non-trivial task.

In Section~\ref{Section:Applications} we apply Theorem \ref{Thm:MainResultRepresentationSetting} to a simple mixture model. There we show the utility of \eqref{eq:MainResultRepresentation} despite the challenges of computing $C_\eta$.
%%%%%%%

\subsection{Classification Setting}

In the classification setting ($\mathcal{T} = [d]$), the set of all (measurable) functions $h:[d]\to[-1,1]$ can be identified with $[-1,1]^d$ via $h \sim v$ if and only if $h(i) = v_i$ for all $i\in[d]$. In this case, expressiveness is not an issue as functions can be easily stored as vectors. For the rest of this section, we assume that the learner can implement any function $h:[d]\to[-1,1]$. The next theorem provides a quantitative version of the bound in \eqref{eq:MainResultPrototype} for the squared-loss $\ell_2$. For ease of notation, we denote $\displaystyle \inf_{h:[d]\to[-1,1]}$ by $\displaystyle \inf_{h}$.

\begin{theorem}
If $\delta>0$, then, for the squared-loss $\ell_2$ and all $n\in\mathbb{N}$, with probability at least $1-\delta$,
\begin{equation}
    \inf_{h} L_{\mathcal{S}_n}(h) - \inf_{h} L(h) \leq 2\sqrt{\frac{2\log(1/\delta)}{n}}.
\end{equation}
\end{theorem}

\begin{proof}
For ease of notation, let $h^* \in \argmin_{h} L(h)$. By the minimality of $\inf_{h} L_{\mathcal{S}_n}(h)$, we have that
\begin{equation}
    \inf_{h} L_{\mathcal{S}_n}(h) - \inf_{h} L(h) \leq L_{\mathcal{S}_n}(h^*) - L(h^*).
\end{equation}
For the functions under consideration, the squared-loss $\ell_2$ is upper bounded by $4$. Hence, Hoeffding's inequality implies that, with probability at least $1-\delta$,
\begin{equation}
    L_{\mathcal{S}_n}(h^*) - L(h^*) \leq 2\sqrt{\frac{2\log(1/\delta)}{n}}.
\end{equation}
The result follows.
\end{proof}

{\bf Remark.} Although our bounds do no depend on $d$, the size of the alphabet $\mathcal{T}$, the complexity of the optimization $\inf_{h}$ does grow with $d$.

The next theorem provides a quantitative version of the bound in \eqref{eq:MainResultPrototype} for the log-loss $\ell_{\textnormal{log}}$.

\begin{theorem}
If $\delta>0$, then, for the log-loss $\ell_\textnormal{log}$ and all $n\geq 4(2d+\log(1/\delta))$, with probability at least $1-\delta$,
\begin{equation}
\label{eq:ClassificationTheoremLogLoss}
    \inf_{h} L_{\mathcal{S}_n}(h) - \inf_{h} L(h) \leq h_b\left(\sqrt{\frac{2d+\log(1/\delta)}{n}}\right),
\end{equation}
where $h_b(x) := -x\log(x)-(1-x)\log(1-x)$ is the so-called binary entropy function.
\end{theorem}

\begin{proof}
For ease of notation, let $\Delta$ be the left hand side of \eqref{eq:ClassificationTheoremLogLoss}. For the log-loss $\ell_\textnormal{log}$, it could be verified that
\begin{equation}
\label{eq:ClassificationLogLossOptimal}
    \inf_{h} L(h) = H(S|T),
\end{equation}
see, e.g., Section~2.2 in \cite{huang2017context}. Indeed, \eqref{eq:ClassificationLogLossOptimal} is attained by the function $h^*(t) = \frac{1}{2}[1+\mathbb{E}(S|T=t)]$. Similarly, it can be shown that
\begin{equation}
    \inf_{h} L_{\mathcal{S}_n}(h) = \hat{H}(S|T),
\end{equation}
where $\hat{H}(\cdot)$ is the plug-in estimate of $H(\cdot)$. Hence,
\begin{equation}
\label{eq:ClassificationConditionalEntropyReduction}
    \inf_{h} L_{\mathcal{S}_n}(h) - \inf_{h} L(h) \leq |\hat{H}(S|T) -H(S|T)|.
\end{equation}
The large deviation bound in Theorem~2.1 in \cite{weissman2003inequalities} shows that, with probability at least $1-\delta$,
\begin{equation}
\label{eq:ClassificationWeissmanInq}
    \textnormal{TV}(\hat{P}_{S,T},P_{S,T}) \leq \sqrt{\frac{2d+\log(1/\delta)}{n}} \leq \frac{1}{2},
\end{equation}
where $\hat{P}_{S,T}$ is the empirical distribution of $S$ and $T$. Thus, \eqref{eq:ClassificationConditionalEntropyReduction} and the inequality by Alhejji and Smith \eqref{inq:AlhejjiSmith} imply that, with probability at least $1-\delta$,
\begin{equation}
\label{eq:ClassificationAlhejjiInq}
    \inf_{h} L_{\mathcal{S}_n}(h) - \inf_{h} L(h) \leq h_b(\textnormal{TV}(\hat{P}_{S,T},P_{S,T})).
\end{equation}
By plugging \eqref{eq:ClassificationWeissmanInq} in \eqref{eq:ClassificationAlhejjiInq}, the result follows.
\end{proof}

{\bf Remark.} Note that in the representation setting ($\mathcal{T}\subseteq\mathbb{R}^q$) under the squared-loss $\ell_2$, $\mathcal{L}$ is equal to the minimum mean square error (MMSE) of $S$ given $T$. Similarly, observe that in the classification setting ($\mathcal{T} = [d]$) under the log-loss $\ell_\textnormal{log}$, $\mathcal{L}$ is equal to the equivocation (conditional entropy) of $S$ given $T$, see \eqref{eq:ClassificationLogLossOptimal}. This highlights the operational interpretation of both losses.
%%%%%%%

\section{Applications and Numerical Experiments}
\label{Section:Applications}

We now consider a specific learning setting and apply Theorem~\ref{Thm:MainResultRepresentationSetting} to illustrate the usefulness of our main results. By explicitly computing $C_\eta$, we show the dependency of the Barron constant on the different elements of the problem.

\subsection{Setting}
\label{Subsection:Setting}

Assume that $S$ is a Rademacher random variable, i.e., $\mathbb{P}(S=1) = \mathbb{P}(S = -1) = 1/2$, and $U$, conditioned on $S$, is a Gaussian random vector $\mathcal{N}(Sv_0,{\rm I}_p)$ for some unit vector $v_0\in\mathbb{R}^p$. We further assume that $T$ is the truncated output of a linear classifier determined by a unit vector $v\in\mathbb{R}^p$, i.e., $T$ is the truncation of $\langle U, v\rangle$ to the interval\footnote{Recall that if $Z \sim \mathcal{N}(0,1)$, then $\mathbb{P}(|Z|>3) < 3\times10^{-3}$. Hence, the proposed truncation is almost immaterial.} $[-3,3]$. Hence, $T$ conditioned on $S$ is a truncated Gaussian random variable $\mathcal{N}(S\mu,1)$ where $\mu := \langle v,v_0 \rangle$. By the symmetry of the Gaussian distribution, it could be verified that, for all $t\in[-3,3]$,
\begin{equation}
    f_{\pm}(t)=\frac{e^{-(t\mp\mu)^2/2}}{\sqrt{2\pi}[\Phi(3+\mu)-\Phi(-3+\mu)]},
\end{equation}
where $\Phi(\cdot)$ is the cumulative distribution function of the standard Gaussian distribution. In order to use the upper bound in \eqref{eq:MainResultRepresentation}, it is necessary to compute $C_\eta$ as introduced in \eqref{eq:ExtensionEta}.

\subsection{Computation of Barron's Constant}

A direct computation shows that we can take $\eta:\mathbb{R}\to\mathbb{R}$ as
\begin{equation}
\label{eq:ExperimentEta}
    \eta(t) = \frac{e^{-(t-\mu)^2/2}-e^{-(t+\mu)^2/2}}{e^{-(t-\mu)^2/2}+e^{-(t+\mu)^2/2}} = \frac{1-e^{-2\mu t}}{1+e^{-2\mu t}} = \tanh(\mu t).
\end{equation}
A direct computation shows that $\eta'(t) = \mu\,\textnormal{sech}(\mu t)^2$. Recall that $\mathcal{F}(h')(\omega) = {\rm i} \omega \mathcal{F}(h)(\omega)$, and hence
\begin{align}
    C_\eta &= \frac{1}{\sqrt{2\pi}} \int_\mathbb{R} |\mathcal{F}(\eta')(\omega)| {\rm d}\omega\\
    \label{eq:RepresentationExampleBarronConstant} &= \frac{\mu}{\sqrt{2\pi}} \int_\mathbb{R} |\mathcal{F}(\textnormal{sech}(\mu \cdot)^2)(\omega)| {\rm d}\omega.
\end{align}
Using contour integration, it can be shown that
\begin{equation}
\label{eq:RepresentationExampleFourierTransform}
    \mathcal{F}(\textnormal{sech}(\cdot)^2)(\omega) = \sqrt{\frac{\pi}{2}} \omega\,\textnormal{csch}\left(\frac{\pi}{2}\omega\right),
\end{equation}
which is non-negative for all $\omega\in\mathbb{R}$. By applying the Fourier inversion formula to \eqref{eq:RepresentationExampleBarronConstant}, we conclude that
\begin{equation}
\label{eq:CetaMixture}
    C_\eta = \mu\,\textnormal{sech}(0)^2 = \langle v,v_0 \rangle.
\end{equation}

{\bf Remark.} Let $g_v:\mathbb{R}^q\to\mathbb{R}$ be the linear classifier $g_v(u) := \langle u,v \rangle$. Recall that, up to an immaterial truncation, $T = g_v(U)$ and $U|S \sim \mathcal{N}(Sv_0,{\rm I}_p)$. Hence, \eqref{eq:CetaMixture} shows the dependency of the Barron constant $C_\eta$ on the linear classifier $g_v$ and the statistical dependency of $S$ and $U$.

\subsection{Theoretical Guarantees}

By \eqref{eq:CetaMixture}, the upper bound in \eqref{eq:MainResultRepresentation} becomes
\begin{equation}
\label{eq:exp-bound}
    \mathcal{L}_{k,\mathcal{S}_n} - \mathcal{L} \leq 2(1+3\langle v,v_0 \rangle)^2\sqrt{\frac{2\log(1/\delta)}{n}} + \frac{36\langle v,v_0 \rangle^2}{k} + \frac{24\langle v,v_0 \rangle}{\sqrt{k}},
\end{equation}
as $\mathcal{K}=[-3,3]$. Observe that if $v$ is parallel to $v_0$, the linear classifier $g_v$ has the potential to leak sensitive attributes, whereas if $v$ is orthogonal to $v_0$, then $g_v$ avoids learning the sensitive attribute as long as the finite learner (attacker, model auditor, adversary, etc) fails in such a task. Our numerical experiments illustrate the validity of this assertion.

It could be verified that, for the squared-loss $\ell_2$, $\mathcal{L} \leq 1$. Indeed, the function $h_0 \equiv 0$ attains $L(h_0) = 1$. Moreover, it can verified that $\mathcal{L} = 1$ if and only if $S$ and $T$ are independent, which in the context of this experiment is equivalent to have $\mu = \langle v,v_0\rangle = 0$. Therefore, a learner with finite expressiveness and statistical capabilities fails whenever $\mathcal{L}_{k,\mathcal{S}_n}$ is close to 1. For $\mu$ small enough, the right hand side of \eqref{eq:exp-bound} is also small, which implies that a learner with arbitrary expressiveness and statistical capabilities will fail if the finite learner fails. In the present context that would imply that the output of the linear classifier $g_v$ reveals little information about $S$.

\subsection{Numerical Results}

In this experiment we assume that the learner has a dataset $\mathcal{S}_n=\{(S_i, T_i): i\in[n]\}$ with $n = 10^5$ i.i.d.~samples drawn from $(S,T)$ as specified in Section~\ref{Subsection:Setting}. The learner adopts a two-layer neural network with $k = 10^3$ neurons in the hidden layer, as formulated in \eqref{eq:DefNN}. Using the sample and this neural network architecture, the learner searches for the model, say $\tilde{h}_k\in\mathcal{H}_k$, that produces the least empirical $\ell_2$-loss during training. Then, the learner approximates $\mathcal{L}_{k,\mathcal{S}_n}$ by $L_{\mathcal{S}_n}(\tilde{h}_k)$. Observe that $\tilde{h}_k$ should be a minimizer of $L_{\mathcal{S}_n}$ in order to have $\mathcal{L}_{k,\mathcal{S}_n} = L_{\mathcal{S}_n}(\tilde{h}_k)$. Although it is known that finding the ERM rule is, in general, an NP-hard problem \cite{ben2003difficulty}, in the present context it is reasonable to assume that $L_{\mathcal{S}_n}(\tilde{h}_k)$ is a good approximation for $\mathcal{L}_{k,\mathcal{S}_n}$ provided that the two-layer neural network is trained exhaustively.

Recall that for the squared-loss, the optimal prediction rule for the true loss is the conditional mean $\eta(t) = \mathbb{E}(S|T=t)$ given in \eqref{eq:ExperimentEta}. Therefore, the minimal true loss is given by
\begin{align}
\label{eq:min-theoretical-loss-InExp}
    \mathcal{L}=\frac{\sqrt{2}}{\sqrt{\pi}p_\mu}\int_{-3}^{3} \frac{e^{-(t+\mu)^2/2}}{1+e^{-2\mu t}} {\rm d}t,
\end{align}
where $p_\mu := \Phi(3+\mu)-\Phi(-3+\mu)$.

Given $\delta=0.01$, \eqref{eq:exp-bound} provides the lower bound for the minimal true loss
\begin{align}
\label{eq:lower_bound-InExp}
\mathcal{L}_{k,\mathcal{S}_n}-\left( 2(1+3\mu)^2 \sqrt{\frac{2\log(100)}{n}} + \frac{(6\mu)^2}{k} + \frac{24\mu}{\sqrt{k}}\right)\leq \mathcal{L}
\end{align}
where the inequality holds with probability no less than $1-\delta=0.99$.

\begin{figure}[t]
	\centering
        \includegraphics[width=0.5\columnwidth]{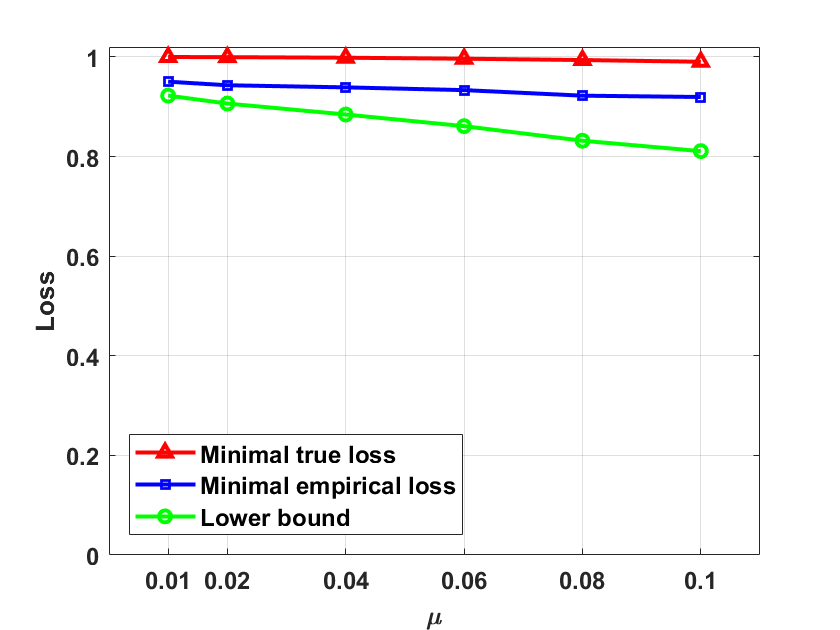}
	\caption{Plot of the minimal true loss $\mathcal{L}$, the minimal empirical loss $\mathcal{L}_{k,\mathcal{S}_n}$ and the lower bound for $\mathcal{L}$ given by Theorem \ref{Thm:MainResultRepresentationSetting} with $\delta=0.01$, $k=10^3$ and $n=10^5$.}
	\label{fig:Exp_MixGaussian}
\end{figure}

In Fig.~\ref{fig:Exp_MixGaussian}, we plot the minimum true loss $\mathcal{L}$ in \eqref{eq:min-theoretical-loss-InExp}, the minimum empirical loss $\mathcal{L}_{k,\mathcal{S}_n}$ and the lower bound in \eqref{eq:lower_bound-InExp} for each $\mu\in\{0.01, 0.02, 0.04, 0.06, 0.08, 0.1\}$. The minimum empirical loss is approximated by $L_{\mathcal{S}_n}(\tilde{h}_k)$ where $\tilde{h}_k$ is the model with the least empirical loss over hundreds of well-trained models\footnote{We train more than 500 randomly initialized models on TensorFlow using the Adam and stochastic gradient descent optimizers for $30$ training epochs. In each trial and each epoch, we calculate the empirical loss of the model over the whole dataset, and choose $\tilde{h}_k$ as the one with the least empirical loss.}. From Fig.~\ref{fig:Exp_MixGaussian}, we observe that (i) the minimal empirical loss is very close to the minimal true loss, and therefore, a two-layer neural network is strong enough for truncated Gaussian mixture data and linear classifiers; (ii) the lower bound for the minimal true loss is about $89\%-97\%$ of the minimal empirical loss. Therefore, a large minimal empirical loss for the simple two-layer neural network learner indicates that the the minimal true loss for an arbitrarily strong learner is still large and the protection of the sensitive variable is guaranteed.
%%%%%%%

\section{Concluding Remarks}

In this paper we provided conditions to guarantee that if a learner with limited expressiveness and statistical capabilities fails to learn a sensitive attribute, then no other learner will succeed. We have also shown through numerical experiments that our main results are meaningful for linear classifiers and data from a Gaussian mixture distribution. In this context, we theoretically found the dependency of the Barron constant on both the classifier and the data distribution. Overall, this work is a step towards model auditing with theoretical guarantees.
%%%%%%%

\subsubsection*{Acknowledgements}

This work was supported in part by the National Science Foundation under Grant CCF-1422358, CCF 1350914, CIF-1422358 and in part by a seed grant towards a Center for Data Privacy from Arizona State University.
%%%%%%%

%\nocite{*}
\bibliographystyle{IEEEtran}
\bibliography{Bibliography}
%%%%%%%
\end{document}